\documentclass{article}

\usepackage{iclr2027_conference}
\usepackage{times}

\usepackage[utf8]{inputenc}

\usepackage{amsmath}
\usepackage{amsthm}
\usepackage{thmtools} %
\usepackage{mleftright}

\makeatletter
\renewenvironment{proof}[1][\proofname]{\par
  \vspace{-\parskip}%
  \vspace{-\topsep}%
  \pushQED{\qed}%
  \normalfont
  \trivlist
  \item[\hskip\labelsep
    \itshape
    #1\@addpunct{.}]\ignorespaces
}{%
  \popQED\endtrivlist\@endpefalse
}
\makeatother

\usepackage{hyperref}
  \hypersetup{hidelinks}
\usepackage[capitalize,nameinlink]{cleveref}
  \crefname{section}{Sec.}{Secs.}
  \crefname{appendix}{App.}{Apps.}
  \crefname{table}{Tab.}{Tabs.}
  \crefname{figure}{Fig.}{Figs.}
  \crefname{equation}{Eq.}{Eqs.}
  \crefname{definition}{Def.}{Defs.}
  \crefname{theorem}{Thm.}{Thms.}
  \crefname{lemma}{Lem.}{Lems.}
  \crefname{proposition}{Prop.}{Props.}
  \crefname{corollary}{Cor.}{Cors.}
  \crefname{claim}{Clm.}{Clms.}

\usepackage{doi}

\newtheorem{theorem}{Theorem}[section]
\newtheorem{lemma}[theorem]{Lemma}
\newtheorem{corollary}[theorem]{Corollary}
\newtheorem{proposition}[theorem]{Proposition}

\theoremstyle{definition}
\newtheorem{definition}[theorem]{Definition} %

\theoremstyle{plain}

\usepackage{amssymb}
\usepackage{mathrsfs}
\usepackage{bm}
\usepackage{cancel}
\usepackage{xcolor}
\usepackage{tikz}
  \usetikzlibrary{positioning}
  \usetikzlibrary{arrows}

\usepackage{pifont}

\usepackage{graphicx}
\graphicspath{ {./images/} }

\usepackage{wasysym}
\usepackage{mathdots}
\usepackage{array}

\usepackage{bbding}
\usepackage{dsfont}

\usepackage{enumitem}
  \setlist{nosep}

\usepackage{natbib}

\usepackage{booktabs}

\usepackage{comment}

\usepackage{framed}
\usepackage{mdframed}

\newcommand{\N}{\mathbb{N}}

\newcommand{\R}{\mathbb{R}}

\newcommand{\macro}[1]{#1}
\newcommand{\paramindices}{K} %
\newcommand{\paramvars}{\macro{\bar{y}}}
\newcommand{\paramvals}{{\macro{\bar{\theta}}}}
\newcommand{\numnewvars}{\kappa}

\newcommand{\sym}[1]{\textnormal{\texttt{#1}}}

\newcommand{\elt}[1]{_{#1}}

\newcommand{\restrict}[2]{\ensuremath{\left.#1\right|_{#2}}}

\newcommand{\VC}{\textnormal{VC}}

\newcommand{\dkey}{d}
\newcommand{\dffn}{d}

\newcommand{\ino}{\in}

\newcommand{\DC}[1]{\textcolor{red}{DC: #1}}

\title{VC Dimension and Expressivity of \\ Real-Valued Transformers}
\author{%
Gavin Dooley, Andy Yang, Yijia Jessica Zhu, David Chiang, Peter Cholak, Anand Pillay  \\
University of Notre Dame \\
\texttt{\{gdooley,ayang4,yzhu26,dchiang,cholak,apillay\}@nd.edu}}

\begin{document}
\iclrfinalcopy
\maketitle

\begin{abstract}
Whereas previous results on abilities and limitations of transformers have restricted the definition of transformers in various ways, here we study softmax-attention, multi-layer transformers operating on real values, with very few additional assumptions. Applying results from real geometry, we obtain upper bounds on the VC dimension and split VC dimension of such transformers ($O(n^4)$ and $O(n^6)$, respectively, where $n$ is the input length). Conversely, we also construct specific transformers witnessing lower bounds on these quantities ($\Omega(n)$ in each case). These results have some notable consequences. For example, within the class of symmetric (permutation-invariant) functions, we show that transformers can uniformly express all functions over an alphabet of one symbol and non-uniformly express all functions over an alphabet of two symbols, but cannot (even non-uniformly) express some functions over an alphabet of six symbols. We also prove limitations on how many bits of a real number a transformer can access.
\end{abstract}

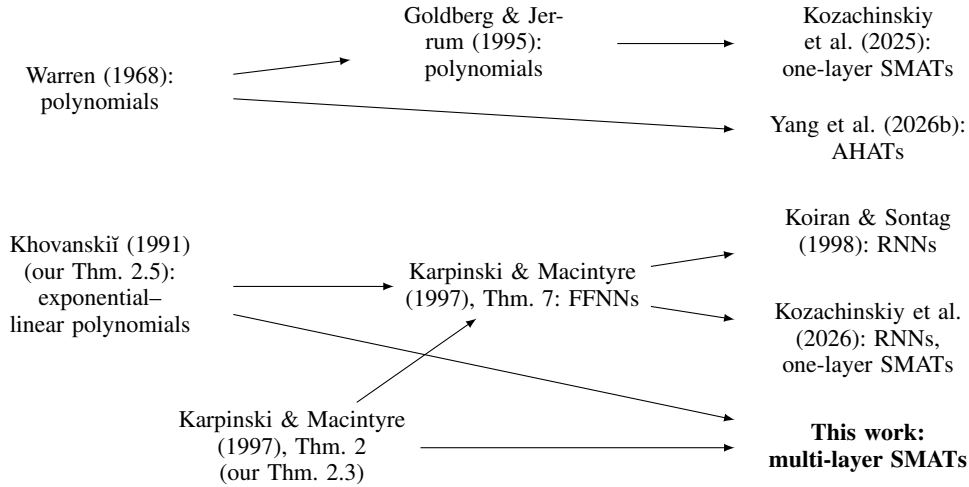
\begin{figure}[b!]
\centering
\small
\begin{tikzpicture}[x=2in,y=0.7in]
\tikzset{every node/.append style={align=center,text width=1.3in}}
\begin{scope}[yshift=1.3in]
\node(warren) at (0,0) {\citet{Warren1968}: \\ polynomials};
\node(gj) at (1,0.35) {\citet{GJ1995}: \\ polynomials};
\node(koz25) at (2,0.35) {\citet{Kozachinskiy2025}: \\ one-layer SMATs};
\node(yang) at (2,-0.35) {\citet{YangSrebroLi2026}: \\ %
AHATs};
\end{scope}
\node(khov) at (0,0.4) {\citet{Khovanskii1991} (our \cref{Khovanskii_theorem}): \\ exponential--linear polynomials};
\node(km2)[text width=1.25in] at (0.5,-0.8) {\citet{KM1997}, Thm.~2 (our~\cref{thm:km1997})};
\node(km7)[text width=1.25in] at (1.1,0.4) {\cite{KM1997}, Thm.~7: FFNNs};
\node(ks) at (2,0.8) {\citet{KOIRAN199863}: RNNs};
\node(koz26) at (2,0) {\citet{Kozachinskiy2026}: RNNs, %
one-layer SMATs};
\node(ours) at (2,-0.8) {\bf This work: \\ multi-layer SMATs};
\begin{scope}[->,>=latex]
\draw (warren) to (gj);
\draw (gj) to (koz25);
\draw (warren) to (yang);
\draw (khov) to (km7);
\draw (khov) to (ours);
\draw (km2) to (km7);
\draw (km7) to (koz26);
\draw (km7) to (ks);
\draw (km2) to (ours);
\end{scope}
\end{tikzpicture}
\caption{Relationship to previous work. FFNN = feedforward neural network, RNN = recurrent neural network, %
AHAT = average-hard attention transformer, SMAT = softmax-attention transformer.
}
\label{fig:overview}
\end{figure}

\section{Introduction}

Previous results on abilities and limitations of  transformers have imposed restrictions of various kinds: on attention (e.g., hard attention instead of soft attention \citep{hahn-2020-theoretical,Perez2021,YangSrebroLi2026}), on numerical precision (e.g., $O(\log n)$ bits where $n$ is the input sequence length \citep{merrill-sabharwal-2023,chen-etal-2025}), or on the number of layers \citep{Kozachinskiy2025,Kozachinskiy2026}. In addition, previous results on limitations of transformers have relied on complexity-theoretic conjectures, most notably $\mathsf{TC}^0 \ne \mathsf{NC}^1$. Indeed, \citet{Kozachinskiy2026} state that it is still an open problem to find any particular task that is not expressible by multi-layer infinite-precision transformers.

In this paper, we study multi-layer softmax-attention transformers (SMATs) operating on real values. Although it has sometimes been argued that some real-valued neural networks can compute any function \citep{siegelmann+sontag:1994,merrill-etal-2022-saturated}, and indeed we prove that real-valued transformers can solve some undecidable problems (\cref{prop:symmetric-size-two,thm:undecidable}), we also prove, to our knowledge for the first time, that there are limitations on their expressivity.

As a real-valued neural network can be regarded as a composition of real functions (possibly preceded or followed by discrete functions like a word or position embedding),
results from real geometry 
have proved useful for studying the limitations of these networks.
\Citet{GJ1995} applied a quantitative bound of \citet{Warren1968} on the number of connected components in a set defined using polynomials to bound the VC dimension of functions defined using polynomials, including many neural networks. 
Although transformers are not polynomials, these results can be applied to transformers in limited ways: \Citet{Kozachinskiy2025} used \citeauthor{GJ1995}'s bound to obtain an inexpressibility result for one-layer transformers, and \citet{YangSrebroLi2026} used Warren's bound to prove VC dimension bounds for hard-attention transformers. %

\Citet{KM1997} extended \citeauthor{GJ1995}'s bound to 
a much broader class of functions. Among other things, they used their theorem, together with a quantitative bound of \citet{Khovanskii1991} on the connected components of a set defined by exponential--linear polynomials,
to bound the VC dimension of feedforward neural networks (FFNNs). Later work applied this result to other kinds of neural networks by (at least indirectly) simulating them by FFNNs: \citet{KOIRAN199863}, to bound the VC dimension of recurrent neural networks (RNNs), and \citet{Kozachinskiy2026}, to prove inexpressibility results for RNNs, state-space models, linear-attention transformers, and one-layer SMATs.

No such simulation by FFNNs is possible in the case of general transformers, which means that analogous bounds for general transformers must directly appeal to the main theorem of \citet{KM1997} in its full generality instead of going through their bound for FFNNs. In this paper, we do just that to obtain a VC dimension bound of $O(n^4)$ and a split VC dimension bound of $O(n^6)$ for multi-layer SMATs, where $n$ is the input length.  Applying this general result entails, as observed by \citet{Kozachinskiy2026},  ``significant additional technical work.'' %
(See \cref{fig:overview} for a summary of how this paper relates to previous work.)

Some particular consequences of these results are:
\begin{itemize}
\item There exist symmetric functions over an alphabet of six symbols that cannot be expressed by any real-valued transformer.
\item Real-valued transformers cannot extract the $k$-th bit of an inputed number if $k \in \omega(n^4)$. Intuitively, although real-valued transformers are allowed infinite precision, they cannot do much with that precision beyond $O(n^4)$ bits.
\end{itemize}

We also provide examples of tasks that transformers can express, complementing the impossibility results summarized above.
\begin{itemize}
    \item There exists a transformer that can compute any symmetric function on an alphabet of two symbols (with a different positional encoding for each symmetric function) (\cref{prop:symmetric-size-two}).
    \item There exists a transformer (with one layer) that can look up the value of the bit at a queried position in an input string (\cref{prop:index-task}). 
\end{itemize}

As we will see below, the two examples above will provide lower bounds on the maximum possible VC dimension of an untrained transformer and split VC dimension of a trained transformer, respectively. We summarize our upper and lower bounds on these quantities in \cref{tab:upper-lower-bounds}.

\begin{table}[h]
    \caption{VC dimension bounds for transformers.}
    \label{tab:upper-lower-bounds}
    \centering
    \medskip
    \begin{tabular}{l|clcl}
        \toprule
        &   \multicolumn{2}{c}{VC dimension}   &  \multicolumn{2}{c}{split-VC dimension}  \\
        &   \multicolumn{2}{c}{(untrained transformers)}   &  \multicolumn{2}{c}{(trained transformers)}  \\        \midrule
        upper bound    &  $O(n^4)$ & (\cref{thm:transformer_vc})  &   $O(n^6)$ & (\cref{cor:transformer_split_vc})   \\
        lower bound     &   $\Omega(n)$ & (\cref{prop:symmetric-size-two})   &     $\Omega(n)$ & (\cref{prop:index-task})   \\
        \bottomrule
    \end{tabular}
\end{table}

\section{Background}

\subsection{Notation}

We write $[n]$ for the set $\{1, \ldots, n\}$ and $\log$ for the base-2 logarithm.

\begin{definition}[Bachmann--Landau notation]
    Given functions $f, g \colon \N \to \N$, we write:
\begin{itemize}
    \item $g \ino O(f)$ if there exist $c$ and $N$ such that, for all $n \geq N$, $g(n) \leq c f(n)$,
\item $g \ino \Omega(f)$ if there exist $c>0$ and $N$ such that, for all $n \geq N$, $c f(n) \leq g(n)$,
\item $g \ino \omega(f)$ if, for all $c>0$, there exists an $N$ such that, for all $n \geq N$, $c f(n) \leq g(n)$, and
    \item $g \ino \tilde{O}(f)$ if there exists a $k$ such that $g\ino O(f\log^kf)$.
\end{itemize}

When the functions $f$ and $g$ are related to
a transformer $T$, we write $g \ino O_T(f)$ to emphasize that the constant $c$ is allowed to depend on the transformer $T$.

\end{definition}

\subsection{Previous results on VC dimension}

\begin{definition}
\label{def:vc}
    Let $X$ be a set, $\mathcal{S}$ a family of subsets of $X$, and $A$ a finite subset of $X$. 
    We say that $\mathcal{S}$ \emph{shatters} $A$ if, for all $A' \subseteq A$, there is a set $S \in \mathcal{S}$ such that $S \cap A = A'$. Define the \emph{VC dimension of $\mathcal{S}$} (denoted $\VC(\mathcal{S})$) to be the greatest $n$ such that $\mathcal{S}$ shatters a set of size $n$ (or $\infty$ if none). %
\end{definition}

Throughout this paper, $X=\R^m$ for some $m$. 
We will be adapting the methods of \citet{KM1997} from FFNNs to transformers.  
We represent an untrained neural network as a finite Boolean combination $\phi({\bar x}, \paramvars)$ of expressions of the form $\tau({\bar x}, \paramvars) = 0$ or $\tau({\bar x}, \paramvars) > 0$, where 
$\tau({\bar x}, \paramvars)$ is a $C^\infty$ function from $\R^{m+k}$ to $\R$. Here, ${\bar x}$ is an $m$-tuple of \emph{input} variables, and $\paramvars$ a $k$-tuple of \emph{parameter} variables.
This $\phi$ gives rise to a family of subsets of $\R^{m}$, indexed by parameters in $\R^k$:
\[\mathcal{S}_\phi = \{ \{{\bar a}\in \R^{m} \mid \phi({\bar a}, \paramvals)\} \mid \paramvals\in \R^{k}\}.\]

In passing, when the $\tau$ are given by polynomials over $\R$, then such ${\mathcal{S}}_{\phi}$ are precisely the {\em semialgebraic families of semialgebraic subsets of $\R^m$}. It is well-known that such families have finite VC dimension \citep{GJ1995}.

Let us first describe Theorem 2 from \citet{KM1997}. For the necessary background regarding submanifolds of $\R^n$, see \cref{sec:background-smooth-manifolds}. Let $\phi({\bar x}, \paramvars)$ be as above, and let $\tau_{1}({\bar x}, \paramvars),\dots,\tau_{s}({\bar x}, \paramvars)$ be the functions appearing in $\phi$.
Define a \emph{p-function} of $\phi$ to be a function $\paramvars \mapsto \tau_i(\bar{a}, \paramvars)$, where $\bar{a} \in \R^m$ and $\tau_i$ is a function appearing in $\phi$.

\begin{theorem} \label{thm:km1997} Given $\phi({\bar x}, \paramvars)$ as above, suppose that $B$ is an integer such that 
for any choice of $r \le k$ many p-functions $F = (F_1, \ldots, F_r)$ of $\phi$, and values 
$(\epsilon_{1},\dots,\epsilon_{r})\in \R^{r}$, if $F^{-1}((\epsilon_{1},\dots,\epsilon_{r}))$ is a $(k-r)$-dimensional submanifold of $\R^{k}$, then it has at most $B$ connected components.  Then the VC dimension of ${\mathcal S}_{\phi}$ is at most $2\log B + (16 + 2\log s)k$. 
\end{theorem}

The application by \citet{KM1997} of \cref{thm:km1997} to FFNNs and our application to transformers (which are both nontrivial) will depend on a choice of $B$ for suitable $\phi$ and $\tau_{i}$ and this will be given by a theorem of \citet{Khovanskii1991}. 

\begin{definition} A \emph{family of $r$ exponential--linear polynomials of degree at most $D$ in $k$ variables and $q$ exponentials} consists of $r$ functions of the form $P_i(\bar{x}, e^{\Lambda_{1}({\bar x})},\dots,e^{\Lambda_{q}({\bar x})}$) where $\bar{x} = (x_1, \ldots, x_k)$, $P_i$ is a polynomial of degree at most $D$,
and the $\Lambda_{j}({\bar x})$ are linear polynomials.
\end{definition}

\begin{theorem}[\citealp{Khovanskii1991}]
    \label{Khovanskii_theorem}
    Let $P_1, \dots, P_r$ and $\Lambda_1, \dots, \Lambda_q$ be a family of exponential--linear polynomials of degree at most $D$ in $k$ variables and $q$ exponentials, where $k \geq r$. Assume that $$\{ \bar{a} \in \mathbb{R}^k \mid P_1(\bar{a}) = 0 \wedge \dots \wedge P_r(\bar{a}) = 0 \}$$ is a submanifold of $\R^k$ of dimension $k - r$. Then the number of connected components of this set is at most 
    $B = 2^{\frac{q(q-1)}{2}} \cdot D^r (rD + k - r + 1)^{k - r} ( ( k - r + 1) (rD + k - r + 1) + r - k )^q$.
\end{theorem}

If $M$ is the largest of $r$, $k$, and $D$ in the theorem statement above, then the bound can be written as $\log B \ino O(q^2 + M\log{M})$.

\subsection{Transformers}

We follow the standard definition of transformer (\cref{sec:transformers}; \citealp{Vaswani2017}).
Our main results (\cref{thm:transformer_vc,cor:transformer_split_vc}) are independent of many details, like attention masking, layer normalization, residual connections, position encodings, and so on. 

An \emph{untrained} transformer $T$ has a fixed number of layers and heads and fixed input/hidden dimension $d$, but depends on unspecified learnable parameters $\paramvars$. %
Let $k$ be the number of such parameters; we may think of $\paramvars$ as belonging to $\R^k$.
A \emph{trained} transformer $T_\paramvals$ is an untrained transformer $T$ together with some fixed parameter values $\paramvals$.

A transformer operates on sequences. We write $T_\paramvals^{(n)}$ for the restriction of $T_\paramvals$ to sequences of length~$n$.
In \cref{def:transformer}, the input is a sequence of real vectors. But in many applications, transformers take discrete inputs and map them to real vectors via word and/or position embeddings (\cref{def:embeddings}). 
Our main results  (\cref{thm:transformer_vc,cor:transformer_split_vc}) hold with or without word/position embeddings, and whether they are considered to be fixed or to vary with the learned parameters. %

\Cref{def:transformer} defines a final output layer, which computes a single real value from the last position.
We say that $T_\paramvals$ accepts a sequence $X$ of length $n$ iff $T_\paramvals^{(n)}(X) > 0$.
Then for any input length $n$, an untrained transformer $T$ gives rise to a family of sets, indexed by parameters $\paramvals$:
\[\mathcal{S}_{T^{(n)}} := \{ \{X \in \R^{n \times d} \mid T_\paramvals^{(n)}(X) > 0\} \mid \paramvals \in \R^k\}. \]
We are interested in the $\VC$ dimension of this family, which we call $\VC(T^{(n)})$ for short.  On general grounds (namely, the o-minimality of the field of real numbers equipped with the exponential function), it is finite.  We will use the methods of \citet{KM1997} to give the bound~$O(n^{4})$.

\section{Main results}

\subsection{VC dimension of transformers}
\label{sec:VC_and_transformers}

\begin{theorem} \label{thm:transformer_vc} Let $T$ be an untrained transformer. Then $\textnormal{VC}(T^{(n)}) \ino O(n^{4})$. %
\end{theorem}

\begin{proof}[Proof sketch]
See \cref{sec:transformer_vc_proof} for the full proof.

When \citeauthor{KM1997} give their VC dimension bound for sigmoid FFNNs, they apply \cref{thm:km1997} to a formula $\tau(\bar{x},\paramvars)>0$ where $\tau$ is the whole network, which is a $C^\infty$ function. 
However, to obtain the bound $B$ by appeal to \cref{Khovanskii_theorem}, a rather subtle argument is used, adding new variables to obtain a system of exponential--linear polynomial equations. 

Unlike an FFNN, a transformer does not compute a $C^{\infty}$-function, due to the appearance of $\max(-, -)$ in the position-wise feed-forward network ${\mathcal F}$ (in each layer). This is actually the only obstruction, as the various rational expressions always have nonzero denominator and the square root in layer normalization is applied to a positive real number bounded away from $0$. So, in the first step, we write a formula $\phi({\bar x}, \paramvars)$, which is equivalent to a transformer's acceptance condition and therefore does not change its VC dimension. It is a Boolean combination of inequalities involving only $C^{\infty}$ functions $\tau_1(\bar{x}, \paramvars), \ldots, \tau_s(\bar{x}, \paramvars)$, where $s \in 2^{O(n)}$, so we can apply \cref{thm:km1997} to it.

It remains to find the bound $B$ on the number of connected components of the submanifold defined by the equations $F_i(\paramvars) = \epsilon_i$ for $i \in [r]$, where the $F_i$ are p-functions of $\phi$. Now \cref{Khovanskii_theorem} requires a system of exponential--linear polynomials, but the $F_i$ are not. 
So, we add some new variables and equations to eliminate square roots (in layer normalization) and division (in attention and layer normalization) and to ensure that the argument of each $\exp$ (in attention) is linear. 
The number of new variables and equations is $\kappa \ino O(n^2)$.
Then \cref{Khovanskii_theorem} gives that the number of connected components of the set defined by the new equations (which is homeomorphic to the set defined by the old equations) is at most $B$,
where $\log B \in O(q^2 + M \log M)$.

Here, $q$ is the total number of exponentials appearing in the system of equations, which is in $O(n^2)$ (one for each pair of positions, for each layer, head, and for each equation).
And $M$ is the greatest among of the number of equations in the system ($r+\kappa$), the number of variables appearing in the system ($k+\kappa$), and the maximum degree of the exponential--linear polynomials used ($4$). Therefore, $\log B\ino O(n^4)$.
By \cref{thm:km1997}, the VC dimension of $T^{n}$ is at most $2 \log B + (16+2\log s)k$ where $s$ is the number of $C^{\infty}$ functions appearing in $\phi$ above. As $\log s \in O(n)$, the VC dimension is in $O(n^4)$, concluding the proof of \cref{thm:transformer_vc}. 
\end{proof}

\subsection{Inputs as parameters}

\Citet{Kozachinskiy2025} devised a way of using VC dimension to study the complexity of a single function defined on sequences (as opposed to a family of such functions). Given a set $S$, let $\Sigma^S$ denote the set of functions from $S$ to $\Sigma$. Identify the set of strings $\Sigma^n$ and the set of functions $\Sigma^{[n]}$ in the natural way. Given $\bar{x} \in \Sigma^{[n] \setminus \paramindices}$ and $\paramvars \in \Sigma^\paramindices$, let $\bar{x} \oplus \paramvars$ denote the unique function (i.e., string) $\bar{z} \in \Sigma^n$ such that $\restrict{\bar{z}}{[n] \setminus \paramindices} = \bar{x}$ and $\restrict{\bar{z}}{\paramindices} = \paramvars$.

\begin{definition}
Given a subset $L^{(n)} \subseteq \Sigma^n$ (thought of as a language in most applications) and a subset $\paramindices^{(n)} \subseteq [n]$, define
\begin{equation*}
\mathcal{S}_{L^{(n)}, \paramindices^{(n)}} := \{ \{\bar{x} \in \Sigma^{[n] \setminus \paramindices^{(n)}} \mid \bar{x} \oplus \paramvars \in L^{(n)} \} \mid  \paramvars \in \Sigma^{\paramindices^{(n)}}\}.
\end{equation*}

Given a trained transformer $T_\paramvals$ and family of subsets of positions $K_n \subseteq [n]$, define
$$\mathcal{S}_{T_\paramvals^{(n)}, \paramindices^{(n)}} := \{ \{ X \in (\R^d)^{[n] \setminus \paramindices^{(n)}} \mid T_\paramvals^{(n)}(X \oplus Y) > 0 \} \mid Y \in (\R^d)^{\paramindices^{(n)}} \}.$$

By considering the quantity $\VC( \mathcal{S}_{L^{(n)},\paramindices^{(n)}} )$, we can measure the complexity of a single function (as opposed to a family of functions). While this quantity depends on a specified partition of the input positions, the following does not.

\end{definition}

\begin{definition}
    Given a subset $L^{(n)} \subseteq \Sigma^n$, define the \emph{split VC dimension of $L$} to be 
    \[\textnormal{Split-VC}(L^{(n)}) := \max_{\paramindices^{(n)} \subseteq [n]} ( \VC ( \mathcal{S}_{L^{(n)}, \paramindices^{(n)}} ) ).\]
\end{definition}

\begin{theorem}
    \label{thm:transformer_position_vc}
    Let $T_\paramvals$ be a trained transformer and let $\paramindices^{(n)} \subseteq [n]$. Then \[\VC\bigl(\mathcal{S}_{T_\paramvals^{(n)}, \paramindices^{(n)}}\bigr)\ino O(n^4 \cdot |\paramindices^{(n)}|^2).\]
\end{theorem}

\begin{proof}
    The proof proceeds exactly as in \cref{thm:transformer_vc}, but now, $|\paramindices^{(n)}|$ of the input vectors become parameters. Each vector has $d$ components, so there are now $k \in O(|\paramindices^{(n)}|)$ parameters. In the application of \cref{thm:km1997}, we must consider up to $r \le k$ many p-functions, so in the application of \cref{Khovanskii_theorem}, the number of exponentials must be multiplied by $k$: Previously we had $q \in O(n^2)$, but now we have $q \in O(n^2 \cdot |\paramindices^{(n)}|)$. Therefore, the VC dimension is in $O(n^4 \cdot |\paramindices^{(n)}|^2)$.
\end{proof}

\begin{corollary} \label{cor:transformer_split_vc}
   Let $T_\paramvals$ be a trained transformer. Then $\textnormal{Split-VC}(T^{(n)}_\paramvals) \ino O(n^{6})$.
\end{corollary}

\begin{proof}
    For every $\paramindices^{(n)} \subseteq [n]$, we have $|\paramindices^{(n)}| \in O(n)$.
\end{proof}

If we try to apply this split VC dimension bound (\cref{cor:transformer_split_vc}), we at once encounter a difficulty.

\begin{proposition}
Let $\Sigma$ be a finite alphabet, and let $L \subseteq \Sigma^*$. For all $n$, $\textnormal{Split-VC}(L^{(n)}) \leq n \log |\Sigma|$.
\end{proposition}

\begin{proof}
    Let $V = \textnormal{Split-VC}(L^{(n)})$, and take a $\paramindices^{(n)} \subseteq [n]$ such that $\VC( \mathcal{S}_{L^{(n)}, \paramindices^{(n)}} ) = V$. By the definition of $\mathcal{S}_{L^{(n)}, \paramindices^{(n)}}$, there must exist at least $2^V$ many distinct tuples $\paramvars \in (L^{(n)})^{\paramindices^{(n)}}$. Since there are no more than $|\Sigma|^n$ tuples $\paramvars \in (L^{(n)})^{\paramindices^{(n)}}$ in total, it follows that $2^V \leq |\Sigma|^n$, which implies that $V \leq n \log |\Sigma|$.
\end{proof}

Since our split VC dimension bound (\cref{cor:transformer_split_vc}) is asymptotically higher than this, if we want to find a task that is not expressible by transformers, we must look beyond languages over finite alphabets. We encounter natural examples of such languages below when studying the extent to which transformers can access the bits of a real number (\cref{sec:bit-indexing-tasks} below).

\section{Symmetric functions}
\label{sec:symmetric}

Given a language $L \subseteq \Sigma^*$, we let $L^{(n)} \subseteq \Sigma^n$ denote the elements of $L$ with length $n$.

\begin{definition}
     A language $L \subseteq \Sigma^*$ is  \emph{symmetric} if, for all $w_0, w_1 \in \Sigma^n$ such that each symbol appears in $w_0$ the same number of times that it appears in $w_1$, we have $w_0 \in L^{(n)}$ if and only if $w_1 \in L^{(n)}$. A function $f \colon \Sigma^n \to \{ 0, 1 \}$ is \emph{symmetric} if it is the characteristic function of a symmetric language. Let $\mathcal{S}^p_n := \{ L^{(n)} \subseteq [p]^n \mid L^{(n)} \text{ is symmetric} \}$. (In the machine learning literature, symmetric functions are sometimes called \emph{permutation equivariant} or \emph{permutation invariant}.)
\end{definition}

\subsection{Non-expressible functions}

\begin{proposition}
    For any $p \geq 1$,  $\textnormal{VC}(\mathcal{S}^p_n) = \binom{n + p -1}{n} \ino \Omega(n^{p-1})$.
\end{proposition}

\begin{proof}
    Let $A$ be a set of strings of such that for any natural numbers $a_0, \dot, a_{p - 1}$ such that $a_0 + \dots + a_{p - 1} = n$, $A$ includes exactly one word $w$ such that, for each $i$, $w$ has exactly $a_i$ many occurrences of the symbol $i$. By definition, the family $\mathcal{S}^p_n$ shatters this set $A$. By the well-known formula for the multiset coefficients \citep[e.g.,][Sec.~1.2]{Stanley2012}, it follows that \begin{equation}\text{VC}(\mathcal{S}^p_n) = \big( \binom{n}{p} \big) = \binom{n + p -1}{n} = \frac{(n + p - 1)!}{n!\, (p - 1)!} \ino \Omega(n^{p-1}).\tag*{\qedhere}\end{equation}
\end{proof}

It follows from \cref{thm:transformer_vc} that there exist symmetric functions on six symbols that cannot be expressed by transformers. A seeming shortcoming of this result is that the functions that it guarantees to be inexpressible are not necessarily computable. We address this critique by pointing out a flipside to this non-uniformity—this inexpressibility result applies to non-uniform transformers (of fixed size) just as well as it does to uniform ones (i.e., it applies to transformers whose parameter values and encoding functions vary with the input length, even uncomputably). Such non-uniform transformers can be used to express uncomputable functions, as we see below in \cref{sec:symmetric_binary}.

\subsection{Alphabet of size \texorpdfstring{$1$}{1}}
\label{sec:symmetric_unary}

\begin{proposition} \label{thm:symmetric_unary}
For any family of (symmetric) functions $f^{(n)} \colon \{\sym{1}\}^n \to \{ 0, 1 \}$, there exists a trained transformer $T_\paramvals$ such that, for any $n$, $T_\paramvals^{(n)}(\sym{1}^n) = f^{(n)}(\sym{1}^n)$.
\end{proposition}

See \cref{sec:transformer-constructions} for the proof. For each (symmetric) function $f^{(n)} \colon \{\sym{1}\}^n \to \{ 0, 1 \}$, the transformer constructed in this proof is uniform in the sense that the same trained transformer, including word and position embeddings, is used for all input lengths. In \cref{sec:symmetric_binary} below, we extend this result from unary to binary languages, but the positional encoding used there depends on the input length.

\subsection{Alphabet of size \texorpdfstring{$2$}{2}}
\label{sec:symmetric_binary}

\begin{proposition}
    \label{prop:symmetric-size-two}
    There exists a trained transformer $T_\paramvals$ such that, for any $n>0$ and symmetric function $f \colon \{ \sym{0}, \sym{1} \}^n \to \{ 0, 1 \}$, there exists a positional encoding $\textnormal{PE} \colon [n] \to \mathbb{R}^d$ such that $T^{(n)}_\paramvals$ with positional encoding $\textnormal{PE}$ computes $f$.
\end{proposition}

See \cref{sec:transformer-constructions} for the proof. In both this construction and the one from \cref{sec:symmetric_unary}, if the symmetric function to be expressed is uncomputable, then the positional encoding used is also uncomputable. In the subsection below, we illustrate the power of real weights themselves by constructing a (real-valued) transformer with a computable positional encoding that expresses an uncomputable function.

\subsection{Expressing an uncomputable function using real weights}

Let $U$ be a universal prefix-free Turing machine (that is, $U$ can simulate any other prefix-free Turing machine, and if $U$ halts on input $x$, then it does not halt on input $xy$ for any nonempty $y$). Chaitin’s~$\Omega$ is defined as the probability that $U$ halts, that is, $\Omega = \sum_{x : \text{$U$ halts on $x$}} 2^{-|x|}$.  
We only need the fact that Chaitin’s $\Omega$ is an uncomputable left-c.e. real number strictly between $0$ and $1$. This means that the language $L = \{ w \in \{ \sym{0}, \sym{1} \}^* \mid \#_\sym{1}(w) / \vert{}w\vert{} < \Omega \}$ is computably enumerable but not computable. 
\begin{proposition} \label{thm:undecidable}
$L$ is recognized by a transformer.
\end{proposition}

See \cref{sec:transformer-constructions} for the proof.

\section{Bit indexing}
\label{sec:bit-indexing-tasks}

Although we have seen an example where infinite precision can be used to express a language that would not be expressible otherwise (\cref{thm:undecidable}), we show here that there are limits to how much a transformer can use a real number as a ``data structure.''

\subsection{Dynamic bit indexing}

In this section, we describe a general family of tasks centered around computing a specific bit of a real number, where both the bit and the real number are given as input. Given functions $p \colon \N \to [n]$ and $\ell \colon \N \to \N$, the input for this task consists of a real number $x \in [0, 1)$ (coded into vectors in the first $p(n)$ input positions) and a bit $b \in [\ell(n)]$ (coded into vectors in the remaining $n - p(n)$ input positions). The intended output of this task is the $b$-th bit of $x$. Our formulation of the task is deliberately agnostic regarding the way that the index and the real are coded as sequences of vectors, since our results hold regardless of the codings used. We will show that, if $\ell(n) \in \omega(n^4 (p(n))^2)$, then this task cannot be solved by a transformer. %

\begin{definition}
    Let $p : \N \to [n]$ and $\ell : \N \to \N$. In the \emph{$p$-partitioned $\ell$-bit index task}, the input is a sequence of real vectors $\mathbf{x}_1, \dots, \mathbf{x}_n \in \R^d$. Let $\langle - \rangle_{\text{real}} : (\R^d)^{p(n)} \to [0, 1)$ and $\langle - \rangle_{\text{index}} : (\R^d)^{n - p(n)} \to [\ell(n)]$ be arbitrary surjections (i.e., arbitrary functions encoding real numbers and indices as sequences of $p(n)$ and $n - p(n)$ real vectors, respectively). Let $x := \langle \mathbf{x}_{1}, \dots, \mathbf{x}_{p(n)} \rangle_{\text{real}}$ and $b := \langle \mathbf{x}_{p(n) + 1}, \dots, \mathbf{x}_{n} \rangle_{\text{index}}$. The intended output of the task is the $b$-th bit of $x$. We denote this task, regarded as a subset of $(\R^d)^n$ in the natural way, by $\mathsf{bit}_{p, \ell}^{(n)}$.
\end{definition}

\begin{proposition}
     Let $p \colon \N \to [n]$, $\ell : \N \to \N$, and $\paramindices^{(n)} = [p(n)]$. Then $\VC( \mathcal{S}_{\mathsf{bit}_{p, \ell}^{(n)}, \paramindices^{(n)}} ) \geq \ell(n)$.
\end{proposition}

\begin{proof}
    For each $x \in [0, 1)$, let $(\mathbf{x}_{1, x}, \dots, \mathbf{x}_{p(n), x}) \in (\R^d)^{p(n)}$ be a sequence of real vectors such that $\langle \mathbf{x}_{1, x}, \dots, \mathbf{x}_{p(n), x} \rangle_{\text{index}} = x$. For each $b \in [\ell(n)]$, let $(\mathbf{x}_{p(n) + 1, b}, \dots, \mathbf{x}_{n, b}) \in (\R^d)^{n - p(n)}$ be a sequence of real vectors such that $\langle \mathbf{x}_{p(n) + 1, b}, \dots, \mathbf{x}_{n, b} \rangle_{\text{index}} = b$. Let $A = \{ (\mathbf{x}_{p(n) + 1, b}, \dots, \mathbf{x}_{n, b}) \mid b \in [\ell(n)] \}$. We claim that $\mathcal{S}_{\mathsf{bit}_{p, \ell}^{(n)}, K}$ shatters this set $A$. For each $A' \subseteq A$, define $\langle A' \rangle_{\text{index}} \subseteq [\ell(n) ]$ by $\langle A' \rangle_{\text{index}} := \{ \langle \mathbf{x}_{p(n) + 1, b}, \dots, \mathbf{x}_{n, b} \rangle_{\text{index}} \mid (\mathbf{x}_{p(n) + 1, b}, \dots, \mathbf{x}_{p(n), b}) \in A' \} $. For each $A'$, let $x_{A'}$ be a real number that has a $1$ at every bit in $\langle A' \rangle_{\text{index}}$ and a $0$ at every bit in $[\ell(n)] \setminus \langle A' \rangle_{\text{index}}$. We have $(\mathbf{x}_{1, x_{A'}}, \dots, \mathbf{x}_{p(n), x_{A'}}, \mathbf{x}_{p(n) + 1, b}, \dots, \mathbf{x}_{n, b}) \in \mathsf{bit}_{p, \ell}^{(n)}$ if and only if $b \in A'$, as needed.
\end{proof}

It follows from \cref{thm:transformer_position_vc} that no transformer can solve the $p$-partitioned $\ell$-bit index task if $\ell(n) \in \omega(n^4 (p(n))^2)$. That is, for any trained transformer $T_\paramvals$ and function $\ell(n) \in \omega(n^4 (p(n))^2)$, there is an $N$ such that, for all $n \geq N$, $T_\paramvals$ does not solve the $\ell(n)$-bit index task for strings of length $n$.

\subsection{Static bit indexing}

Indeed, we show that, if the real number $x$ is input directly in the first position, then no transformer can index the $b(n)$-th bit of $x$ when $b(n) \in \omega(n^4)$ even if $b$ is ``hard-coded'' into the transformer instead of being taken as part of the input.

\begin{definition}
Let $b \colon \N \to \N$ be any function mapping sequence lengths to indices. In the \emph{hard-coded $b(n)$-th-bit index task}, the input is a sequence of real numbers $x_1, \ldots, x_n$, where $0 \le x_1 < 1$. The intended output is the $b(n)$-th bit of $x_1$.
\end{definition}

\begin{proposition}
\label{prop:hard-coded-index}
Let $b(n) \in \omega(n^4)$. For any trained transformer $T_\paramvals$, there exists a length $N$ such that for all $n\ge N$, $T_\paramvals^{(n)}$ does not solve the hard-coded $b(n)$-th-bit index task for length $n$. 
\end{proposition}

\begin{proof}[Proof sketch]
    See \cref{sec:hard-coded-index-proof} for the full proof.
    
    Our proof uses similar ideas to our proof of \cref{thm:transformer_vc} but has a fundamentally different structure. In this proof, we do not use the concept of VC dimension and instead directly use \cref{Khovanskii_theorem} to prove that, for any fixed inputs $v$, the following set has $2^{O(n^4)}$ connected components:
 \[
    S_{\paramvals,v}^{(n)}
    :=
    \{
    x\in[0,1) \mid
    T_\paramvals^{(n)}
    (
    x, v
    )>0
    \}.
    \]
    Importantly, we cannot simply assume that $S_{\paramvals,v}^{(n)}$ is a $(k-r)$-dimensional submanifold of $\R^k$. In fact, it is not. In order to apply \cref{Khovanskii_theorem}, we construct a set that is a $(k-r)$-dimensional submanifold of $\R^k$ and has no fewer connected components than $S_{\paramvals,v}^{(n)}$. The essential tools used in this construction are the regular level set theorem and Sard's theorem (\cref{thm:regular-level-set,thm:sard}), and our proof depends crucially on the fact that the real number $x$ is entered directly as a single coordinate of the input.
\end{proof}

\subsection{Discussion}

Even though real-valued transformers have an infinite number of bits at their disposal, and can make use of this infinite precision in some cases (\cref{thm:undecidable}), the results in this section show that they can make only limited use of this precision.

These results rule out the use of real numbers as ``data structures'' that can store an unbounded or infinite amount of information. For example, \citet[Theorem~1]{merrill-etal-2022-saturated} argue that an average-hard attention transformer with rational values can decide any language, by encoding an entire string of $n$ bits into a single rational number and using a position-wise FFNN to decide whether the string belongs to the language. But by \cref{thm:transformer_vc}, transformers cannot compute all such functions, as this family of functions has VC dimension $2^n$.  %

Although our results apply to a wide range of transformer variants, they only apply to transformers defined using a certain set of operations. Not included among these operations is rounding.
So a possible criticism is that real-world transformers (or the hardware they run on) do in fact use rounding, raising the possibility that they could exploit rounding to solve problems that idealized, real-valued transformers cannot.
However, the fact that language models are routinely trained at higher precision and then quantized to lower precision without catastrophic loss of accuracy \citep[e.g.,][]{frantar2023optq} shows that real-world transformers do not exploit rounding in this way.

Another possible criticism of the inexpressibility results proven in this section is that each of our bit indexing tasks, when regarded as a language, requires an alphabet whose size grows unboundedly with $\ell(n)$. Expanding on a footnote of \citet{hahn-2020-theoretical}, \citet{chen-etal-2025} prove that any symmetric circuit of depth $L$ can be simulated by a transformer of depth $6L$. Consequently, if we could obtain a sufficiently strong transformer lower bound (an inexpressibility result, in our terminology) for a task over a fixed finite alphabet, their simulation would yield a corresponding lower bound for symmetric circuits. Obtaining such a lower bound would constitute a breakthrough in circuit complexity: in particular, no superlinear wire lower bound is known for general symmetric circuits of depth $3$ or greater \citep{chen-etal-2025}.

\section{Positional indexing}\label{sec:position-indexing-tasks}

Next, we consider the task of selecting the $k$-th \emph{position} of the sequence (as opposed to the $k$-th bit, as above).
Using this task, we get an $\Omega(n)$ lower bound on the split VC dimension of transformers.

\begin{definition}
In the \emph{index task}, which we write as $\mathsf{ind}_n$, the input is $x_1 \cdots x_{n-1} k$, where $x_1, \dots, x_{n - 1} \in \{ 0, 1 \}$ and $k \in [n-1]$. The intended output is the value of the bit $x_k$.
\end{definition}

\begin{proposition} \label{thm:index-task-vc}
$\textnormal{split-VC}(\mathsf{ind}_n) \geq n-1$.
\end{proposition}

\begin{proof}
Let $\paramindices = \{n\}$ be the set of positions to treat as parameters, 
so we want to show $\VC(\mathcal{S}_{\mathsf{ind}_n, \paramindices}) \ge n-1$.
Let $A = [n-1]$ be the set to be shattered.
    For any subset $A' \subseteq A$, let $w_{i, A'} = 1$ if and only if $i \in A'$. We have $\mathsf{ind}_n(w_{1, A'}, \dots, w_{n-1, A'}, k) = 1$ if and only if $k \in A'$. It follows that $\VC(\mathcal{S}_{\mathsf{ind}_n,K }) \geq n-1$, as needed.
\end{proof}

\begin{proposition}
\label{prop:index-task}
    The index task is recognizable by a one-layer transformer.
\end{proposition}

\begin{proof}
We set up the word and position embeddings so that
\begin{align*}
X_{i,*} &= \begin{bmatrix}
x_i &
0 &
i &
i^2
\end{bmatrix}^\top \quad (i \in [n-1]) &
X_{n,*} &= \begin{bmatrix}
0 &
k &
n &
n^2
\end{bmatrix}^\top
\end{align*}
The first self-attention layer computes queries, keys, and values
\begin{align*}
W^{\textnormal{Q}} X_{n,*} &= \begin{bmatrix}
4k\gamma & -2\gamma
\end{bmatrix}^\top
&
W^{\textnormal{K}} X_{j,*} &= \begin{bmatrix}
j & j^2
\end{bmatrix}^\top
&
W^{\textnormal{V}} X_{j,*} &= \begin{bmatrix}
x_j \end{bmatrix}.
\end{align*}
The attention scores at the last position are $s_{nj} = 4\gamma kj-2\gamma j^2$. So position $k$ has the highest score and decreases by at least $\gamma$ for each position in either direction. Then, as in the proof of \cref{thm:symmetric_unary}, we can make the attention output positive iff $x_k = 1$.
\end{proof}

This result may appear to contradict Theorem 3.4 of \citet{Kozachinskiy2025}, which states that no one-layer transformer can solve the index task. This apparent contradiction is resolved by observing out that the inexpressibility result of \citet{Kozachinskiy2025} applies to a specific version of the index task where the output is computed at a blank token appended to the input.

\section{Generalization bounds}

Our VC dimension bounds imply generalization bounds for multi-layer, real-valued SMATs under very weak restrictions. 
Using \cref{thm:transformer_vc}, we can bound the generalization error of such a transformer by 
$\tilde{O}(n^2/\sqrt{N})$, where $N$ is the training data size \citep{ss-bendavid}. %
\Citet{YangSrebroLi2026} give a logarithmic upper-bound on the VC dimension of AHATs (which crucially do not use $\exp$), implying a generalization bound of $\tilde{O}(\sqrt{\log n/N})$.

\citet{pmlr-v162-edelman22a} used a covering number bound to derive the first generalization bound for transformers, %
of $\tilde{O}(\sqrt{\mathsf{poly}(C)\log(n)/N})$, where $C$ is a function of the norms of the attention and FFNN matrices.  
\citet{pmlr-v238-trauger24a}  also derived a covering number generalization bound, independent of $n$. 
For \citet{pmlr-v162-edelman22a}, the bound for a single attention head would scale as $\tilde{O}(\sqrt{C_X^2\mathsf{poly}(C)\log(n)/N})$, where $\|X\|_{2,\infty} < C_X$ bounds the norms of the inputs. The bounds of \citet{pmlr-v238-trauger24a} also depend on $C_X$.
Thus, when $C_X$ grows  sufficiently fast with $n$, our bound is asymptotically stronger than these bounds. \Citet{li2026sharpergeneralizationboundstransformer} proved generalization bounds for cases where the input norms are unbounded, but they still make assumptions about the distributions  the inputs are drawn from. Our bound is asymptotically stronger than theirs when the parameters of the input distributions grow sufficiently fast with $n$ or when no such assumptions can be made.

\section{Future research}

\Cref{thm:transformer_vc,cor:transformer_split_vc} give upper bounds on the VC dimension of an (untrained) transformer and the split VC dimension of a (trained) transformer, respectively, and \cref{prop:symmetric-size-two,prop:index-task} give lower bounds on both. These upper and lower bounds are listed in \cref{tab:upper-lower-bounds}. There are clearly gaps between our upper and lower bounds. While it would certainly be desirable to close these gaps, it is worth noting that \citet[Sec.~5]{KOIRAN199863} leave it as an open problem to close similar gaps for RNNs, and their problem is still open to the best of our knowledge. %

Another direction for further research is to explore the extent to which our theoretical results manifest in practice. Experiments regarding the abilities of transformers to learn symmetric or indexing functions or to access specific bits of real numbers in other ways would shed light on the practical implications of our results.

\section*{Use of artificial intelligence}

We used artificial intelligence to review literature and polish writing. We did not use artificial intelligence to generate proofs.

\bibliography{references}

\clearpage

\crefalias{section}{appendix}
\appendix

\section{Definition of Transformers}
\label{sec:transformers}

We follow, for the most part, notational conventions used by \cite{Strobl2024}.

\begin{definition}[\citealp{Vaswani2017}] \label{def:transformer}
    Fix an input--output dimension $d \in \N$.

    An \emph{attention head} is a family of functions, indexed by input sequence length $n$, of the form 
    \begin{align}
        \mathcal{H}^{(n)} \colon \R^{n\times d} &\to \R^{n \times d} \notag \\
        s_{i,j} &= \frac{1}{\sqrt{d}} (W^{\textnormal{Q}}X\elt{i,*}) \cdot (W^{\textnormal{K}}X\elt{j,*}) \notag \\
        \alpha_{i,j} &= \left[\textnormal{softmax} s_{i,*}        
        \right]\elt{j} \notag \\
        &= \frac{\exp s_{i,j} }{\sum_{j' \in [n]}\exp s_{i,j'}} \label{eq:softmax} \\
        [\mathcal{H}^{(n)}(X)]\elt{i} &= \sum_{j \in [n]} \alpha_{i,j} (W^{\textnormal{V}}X\elt{j,*})
        \label{eq:att_head}
    \end{align}
    with 
    parameters $W^{\textnormal{Q}} \in \R^{\dkey \times d}$, $W^{\textnormal{K}} \in \R^{\dkey \times d}$, and $W^{\textnormal{V}} \in \R^{d \times d}$.
    In many cases, $\mathcal{H}$ is \emph{masked} so that $[\mathcal{H}(X)]\elt{i}$ only depends on $X\elt{j}$ for $j\le i$. The presence or absence of masking makes no difference to our results.)
    
    An \emph{attention sublayer} is a family of functions of the form
    \begin{align}
        \mathcal{A}^{(n)} \colon \R^{n\times d} &\to \R^{n \times d} \notag \\
        X &\mapsto \sum_{h = 1}^{H} \mathcal{H}^{(n)}_{h}(X)
        \label{eq:attn}
    \end{align}
    with some number of heads $H$ and an attention head $\mathcal{H}_h$ for each $h \in [H]$.
    
    A \emph{layer normalization} is a function
    \begin{align}
        \mathcal{N} \colon \R^d &\to \R^d \notag \\
        \mathbf{x} &\mapsto \gamma \odot \frac{\mathbf{x} - \overline{\mathbf{x}}}{\sqrt{\text{var}(\mathbf{x}) + \varepsilon}} + \beta, \label{eq:layer_norm}
    \end{align}
    with some numerical stabilizer $\varepsilon > 0$ and parameters $\gamma \in \R^d$ and $\beta \in \R^d$. Here, $\odot$ denotes element-wise multiplication, $\overline{\mathbf{x}}$ denotes the mean of $\mathbf{x}$ ($\overline{\mathbf{x}} := \frac{1}{d} \sum_{i \in [d] } \mathbf{x}_i$) and $\text{var}(\mathbf{x})$ denotes the variance of $\mathbf{x}$ ($\text{var}(\mathbf{x}) := \frac{1}{d} \sum_{i \in [d]} (\mathbf{x}_i - \overline{\mathbf{x}})^2$).

    A \emph{position-wise feed-forward network} is a function of the form
    \begin{align}
        \mathcal{F} \colon \R^d &\to \R^d \notag \\
    \mathbf{x} &\mapsto W_2 (\max(0,W_1\mathbf{x} + \mathbf{b}_1)) + \mathbf{b}_2, \label{eq:ffn}
    \end{align}
    with 
    parameters $W_1 \in \R^{\dffn \times d}$, $\mathbf{b}_1 \in \R^{\dffn}$, $W_2 \in \R^{d \times \dffn}$, and $\mathbf{b}_2 \in \R^{d}$.
    The $\max$ is taken element-wise.

    Both $\mathcal{N}$ and $\mathcal{F}$ can be extended position-wise to functions from $\R^{n \times d}$ to $\R^{n \times d}$.

    A \emph{transformer layer} is a family of functions of the form
    \begin{align*}
        \mathcal{L}^{(n)} \colon \R^{n\times d} &\to \R^{n \times d} \\
        X &\mapsto \mathcal{N}_2\biggl(\mathcal{N}_1\bigl(X + \mathcal{A}^{(n)}(X)\bigr) + \mathcal{F}\Bigl(  \mathcal{N}_1\bigl(X + \mathcal{A}^{(n)}(X)\bigr)\Bigr)\biggr),
    \end{align*} where $\mathcal{A}^{(n)}$ is an attention sub-layer, $\mathcal{F}$ is a position-wise feed-forward network, and $\mathcal{N}_1$ and $\mathcal{N}_2$ are layer normalizations. We have shown the original definition, which places layer normalizations after their respective sub-layers. Other orderings are possible, and would make no difference to our results.

    In this paper, we treat transformers as binary classifiers.
    So an \emph{output layer} is a family of functions of the form
    \begin{align*}
        \mathcal{O}^{(n)} \colon \R^{n \times d} &\to \R \\
        X &\mapsto W X_{n,*} + b
    \end{align*}
    with parameters $W \in \R^{1 \times d}$ and $b \in \R$.

    A \emph{transformer} is a family of functions of the form
    \begin{align*}
        T^{(n)} \colon \R^{n\times d} &\to \R \\
        X &\mapsto \mathcal{O}^{(n)}(\mathcal{L}^{(n)}_L(\dots(\mathcal{L}^{(n)}_1(X))\dots)),
    \end{align*}
    with some depth $L$, where $\mathcal{L}_1^{(n)}, \dots, \mathcal{L}_L^{(n)}$ are transformer layers, and $\mathcal{O}^{(n)}$ is an output layer.
\end{definition}

\begin{definition}
\label{def:embeddings}
    Given an alphabet $\Sigma$, a \emph{word embedding} is a function  $\text{WE}\colon \Sigma \to \R^d$. A \emph{position embedding} is a family of functions $\text{PE}^{(n)} \colon [n] \to \R^d$. An \emph{embedding} is a family of functions of the form
    \begin{align*}
        e^{(n)} \colon \Sigma^n &\to (\R^d)^n \\
        e^{(n)}(w)\elt{i} &:= \text{WE}(w\elt{i}) + \text{PE}^{(n)}(i),
    \end{align*}
    where $\text{WE}\colon \Sigma \to \R^d$ is a word embedding and $\text{PE}^{(n)} \colon [n] \to \R^d$ is a position embedding. A \emph{discrete-input transformer} is a family of functions of the form
    \begin{align*}
        \mathcal{T}^{(n)} \colon \Sigma^n &\to \R \\
        w &\mapsto T^{(n)}(e^{(n)}(w)),
    \end{align*}
    where $T$ is a (real-input) transformer and $e$ is an embedding.
\end{definition}

\section{Background on submanifolds of $\R^n$}
\label{sec:background-smooth-manifolds}

\begin{definition}[Some basic topology of $\R^n$]
    For $r > 0$ and $\mathbf{p} \in \R^n$, define the \emph{open ball of radius~$r$ and center $\mathbf{p}$} to be the set $B_r(\mathbf{p}) := \{ \mathbf{x} \in \R^n : || \mathbf{x} - \mathbf{p} || < r \}$. A set $U \subseteq \R^n$ is called \emph{open (in $\R^n$)} if, for any $\mathbf{p} \in U$, there exists $r > 0$ such that $B_r(\mathbf{p}) \subseteq U$. Given a subset $X \subseteq \R^n$, a set $U \subseteq X$ is called \emph{open in $X$} if, for any $\mathbf{p} \in U$, there exists $r > 0$ such that $B_r(\mathbf{p}) \cap X \subseteq U$. When there is no risk of ambiguity, we refer to sets that are open in $U$ as \emph{open subsets of $U$}. A set $X \subseteq \R^n$ is called \emph{connected} if it cannot be expressed as the disjoint union of two non-empty open subsets of itself. A \emph{connected component of $X$} is a maximal connected subset of $X$ (i.e., a connected subset $C \subseteq X$ such that, for any connected subset $C \subseteq C' \subseteq X$, we have $C = C'$).
\end{definition}

Intuitively, a $k$-dimensional submanifold of $\R^n$ is a subset of $\R^n$ that locally resembles $\R^k$. 

\begin{definition}[Diffeomorphisms and submanifolds of $\R^n$]
    Given an open set $U \subseteq \R^n$, a function $f : U \to \R^n$ is called \emph{$C^m$} if all of its $m$-th derivatives exist and are continuous, and is called \emph{$C^\infty$} (or \emph{smooth}) if it is $C^m$ for every $m$. Given an open subset $U \subseteq \R^n$, a function $f : U \to \R^n$ is called a \emph{diffeomorphism} if $f$ and $f^{-1}$ are both $C^\infty$, $f$ is injective, and $f(U)$ is open in $\R^k$. If $k \leq n$, define $\R^k \times \mathbf{0} := \{ (x_1, \dots, x_n) \in \R^n \mid x_i = 0 \text{ for all } k < i \leq n \}$. A set $X$ is called a \emph{$k$-dimensional submanifold of $\R^n$} if, for every $\mathbf{p} \in X$, there exists an open $U \subseteq \R^n$ with $\mathbf{p} \in U$ and a diffeomorphism $f : U \to \R^n$ such that $f(X \cap U) = \R^k \times \mathbf{0} \cap f(U)$.
\end{definition}

\begin{definition}[Regular and critical values]
    Let $U \subseteq \R^n$ be an open set, and let $f : U \to \R^m$ be a smooth function. Let $df$ denote the differential of $f$, that is, the matrix $\begin{bmatrix} \frac{\partial f}{\partial x^1} & \cdots & \frac{\partial f}{\partial x^n} \end{bmatrix}$, where $x^i$ is the $i$th coordinate of $\R^n$ and each partial derivative is regarded as a column of length $m$. Given a point $\mathbf{x} \in U$, the evaluated differential $df(\mathbf{x})$ is identified with the linear function from $\R^n$ to $\R^m$ that it induces via matrix--vector multiplication. A point $\mathbf{x} \in U$ is called a \emph{regular point of $f$} if the map $df(\mathbf{x}) : \R^n \to \R^m$ is surjective, and called a \emph{critical point of $f$} otherwise. A point $\mathbf{y} \in \R^m$ is called a \emph{regular value of $f$} if every point in $f^{-1}(\mathbf{y})$ is regular, and called a \emph{critical value of $f$} otherwise. For any $\mathbf{y} \in \R^m$, the set $f^{-1}(\mathbf{y})$ is called a \emph{level set of $f$}. A level set $f^{-1}(\mathbf{y})$ is called a \emph{regular level set of $f$} if $\mathbf{y}$ is a regular value of $f$.
\end{definition}

We state the following two theorems only for our special case of submanifolds of $\R^n$, but their natural generalizations hold for arbitrary smooth manifolds.

\begin{theorem}[Regular level set theorem] \label{thm:regular-level-set}
    Let $X$ be $k$-dimensional submanifold of $\R^n$, let $Y$ be a $j$-dimensional submanifold of $\R^m$, and $f : X \to Y$ be a smooth function. Every regular level set of $f$ is a $(k - j)$-dimensional submanifold of $\R^m$.
\end{theorem}

\begin{definition}[Measure in a submanifold]
    A set $A \subseteq \R^k$ has \emph{($k$-dimensional) measure zero} if, for any $\delta > 0$, $A$ can be covered by a countable collection of $k$-dimensional hyperrectangles the sum of whose volumes is less than $\delta$. Let $X$ be a $k$-dimensional submanifold of $\R^n$, and let $A \subseteq X$. We say that $A$ has \emph{($k$-dimensional) measure zero in $X$} if, for any open $U \subseteq \R^n$ and any diffeomorphism $f : U \to \R^n$ such that $f(X \cap U) = \R^k \times \mathbf{0} \cap f(U)$, the image $f(A \cap U)$ has ($k$-dimensional) measure zero when regarded as a subset of $\R^k$.
\end{definition}

\begin{theorem}[Sard's theorem] \label{thm:sard}
    Let $X$ be $k$-dimensional submanifold of $\R^n$, let $Y$ be a $j$-dimensional submanifold of $\R^m$, and let $f : X \to Y$ be a smooth function. The set of critical values of $f$ has measure zero in $Y$.
\end{theorem}

\section{Proof of \cref{thm:transformer_vc}}
\label{sec:transformer_vc_proof}

When \citeauthor{KM1997} give their VC dimension bound for FFNNs, they apply \cref{thm:km1997} directly to a sigmoid FFNN as a single function $f$ (so $s=1$). This is possible because a sigmoid FFNN is a $C^\infty$ function. 
The value $B$ is computed by appeal to \cref{Khovanskii_theorem}, but to use this theorem  a rather subtle argument is used, adding new variables 
to obtain a suitable system of exponential--linear polynomial equalities and inequalities. 

Let $T$ be our given (untrained) transformer, with $d$ the dimension of the vectors, and $L$ the number of layers, and $H$ the number of heads. We let $\paramvars$ be the tuple of parameters. We fix an input length $n$. So the input is in $\R^{nd}$.  Let us write $f({\bar x}, \paramvars)$ for the function computed by $T^{(n)}$, on input ${\bar x}$ and with parameters $\paramvars$. We will write $f$ explicitly as a composition of the various functions in  \cref{def:transformer} over the $L$ layers. 
So we could think of $f({\bar x}, \paramvars)$ as a ``term" and the VC-dimension of $T^{(n)}$ is the VC dimension of the expression $f({\bar x}, \paramvars) > 0$. 

But $f$ is \emph{not} a $C^{\infty}$-function, due to the appearance of $\max(-, -)$ in the position-wise feed-forward network ${\mathcal F}$ (in each layer).  This is actually the only obstruction, as the various rational expressions always have nonzero denominator and the square root in layer normalization is applied to a positive real number bounded away from $0$.

So what we will do in the first step is replace the ``formula" $f({\bar x}, \paramvars) > 0$ by an equivalent formula  $\phi({\bar x}, \paramvars)$ involving a certain number of inequalities  $f_{\eta}({\bar x}, \paramvars) > 0$ built by making ``substitutions" in $f>0$, and an additional finite number of inequalities where all the functions involved, including the $f_{\eta}$,  are $C^{\infty}$. So we can apply \cref{thm:km1997} to $\phi({\bar x}, \paramvars)$.

We define the $f_{\eta}$ and the additional inequalities inductively.  As mentioned above, $f$ is a composition of terms representing the various stages. Consider the first time $\max$ appears in the first layer. Remember $\max$ is the pointwise maximum of $0$ and a certain $d$-dimensional column vector, and $n$ of these appear. So inside $f$, we have $\max(0,\tau_{i,j})$ for $i=1,\dots,n$, $j=1,\dots,d$, and $\tau_{i,j}$ are terms built so far and stand for $C^\infty$ functions.
Let $\eta \in \{0,1\}^{[n]\times[d]}$. Let $f_{\eta}$ be the result of replacing each occurrence of $\max(0,\tau_{i,j})$ in $f$ by $\tau_{i,j}$ if $\eta(i,j)=1$, and by $0$ if $\eta(i,j)=0$.
Let~$g_{\eta}$ be the conjunction of  $f_{\eta} > 0$ with  the set of inequalities $\tau_{i,j} \geq 0$ whenever $\eta(i,j) = 1$ and $\tau_{i,j} < 0$ whenever $\eta(i,j) = 0$ (for $i \in [n]$ and $j \in [d]$). And let $\phi_{1}({\bar x}, \paramvars)$ be the disjunction of the $g_{\eta}$ for $\eta\in \{0,1\}^{[n]\times[d]}$, which is equivalent to $\phi({\bar x}, \paramvars)$ but has fewer occurrences of $\max(0, -)$. 
Now iterate this process through the layers. Let us briefly describe what to do at layer 2. Fix $\eta$, and 
consider again where the $\max$ from layer~2 appears in $f_{\eta}$.  And again we will have inside $f_{\eta}$, $\max(0,\pi_{i,j})$ for $(i,j)\in nd$. Choose $\nu\in \{0,1\}^{[n]\times[d]}$, and let $f_{(\eta, \nu)}$ result from $f_{\eta}$ by replacing $\max(0,\pi_{i,j})$ by $\pi(i.j)$ if $\nu(i, j) = 1$ and by $0$ if $\nu(i,j) = 0$.  And let $g_{(\eta, \nu)}$ consist of the conjunction of $f_{(\eta,\nu)} > 0$ with the rest of the inequalities in $g_{\eta}$ and the inequalities  $\pi_{i,j} \geq 0$ whenever $\nu(i.j) = 1$ and $\pi_{i,j} < 0$ whenever $\nu(i,j) < 0$.  Let $\phi_{2}({\bar x}, \paramvars)$ be the disjunction of the $g_{(\eta,\nu)}$ as $\eta, \nu$ vary.
Continue in the same way through  the layers and let $\phi({\bar x}, \paramvars) = \phi_{L}({\bar x}, \paramvars)$. Then notice that $\phi$ is the disjunction of $(2^{nd})^{L}$ expressions $g_{\alpha}$ say, each consisting of the conjunction of some $f_{\alpha} >0$ (arising from $f$ by making the appropriate substitutions) and a further $Lnd$ inequalities, where all the functions (or terms) appearing are $C^{\infty}$ (in $({\bar x}, \paramvars)$).  So $\phi({\bar x}, \paramvars)$ is a Boolean combination of (in fact a disjunction of conjunctions of)  $s = (2^{nd})^{L}(Lnd + 1)$ inequalities $\tau({\bar x}, \paramvars) > 0$, $\tau({\bar x}, \paramvars) \geq 0$, for $\tau$ a $C^{\infty}$ function. 
Let us write $\{\tau_{i} \mid i \in [s]\}$ for these $\tau$ appearing in $\phi$. 

So we can apply \cref{thm:km1997} and it remains to find the bound $B$ from there. 
Let $k$ be the dimension of the parameter space $\paramvars$.  Let ${\bar a}_{1}, {\bar a}_{2},\dots,{\bar a}_{V}$ be $nd$-tuples.  Fix $r\leq k$ and let $F_{1},\dots,F_{r} \colon \R^{k}\to \R$ be $r$ many p-functions from among the $\tau_i({\bar a}_{j}, \paramvars)$. Let $\epsilon_{1},\dots,\epsilon_{r} \in \R$ and suppose that the preimage of $(\epsilon_{1},\dots,\epsilon_{r})$ under $(F_{1},\dots,F_{r})$ is a $(k-r)$-dimension $C^{\infty}$-submanifold of $\R^{k}$. We want to bound the number of connected components of this submanifold (independent of the choices of the data $\bar{a}_j$).

We now add some new variables and defining equations so as to be able to apply \cref{Khovanskii_theorem}.  Now for each $\tau({\bar x}, \paramvars)$ such that for some ${\bar a}_{j}$, $\tau({\bar a}_{j}, \paramvars)$ is among the $F_{i}(\paramvars)$, $\tau$ represents all or some initial part of a run of the transformer (with input ${\bar x}$, parameters $\paramvars$).  Consider the first time that softmax appears in some such $\tau$ (so in the first layer).

We consider the softmax expression (\cref{eq:softmax}) as a function of the input variables ${\bar x}\in \R^{nd}$ and parameters $\paramvars$.  For $1\leq i,j\leq n$, add new variables $z^{\textnormal{s}}_{i.j}$ and $z^{\textnormal{a}}_{i,j}$,  
replace $\alpha_{i,j}$ in \cref{eq:att_head} with new variables $z^{\textnormal{a}}_{i,j}$, and add conditions $z^{\textnormal{s}}_{i,j} = s_{i,j}$ (where note the latter is a linear polynomial in ${\bar x}$ and $\paramvars$) as well as:
\begin{align*}
z^{\textnormal{a}}_{i,j} \left(\sum_{\ell} \exp z^{\textnormal{s}}_{i,\ell}\right) &= \exp z^{\textnormal{s}}_{i,j}.
\end{align*}
Now consider the first time that layer normalization appears in some $\tau$.  This involves taking square roots of terms $v_{i}({\bar x}, \paramvars)$ which already appear in $\tau$ and whose value is positive and bounded away from $0$.  Add new variables $z^{\textnormal{s}}_{i}$, and equations 
$(z^{\textnormal{s}}_{i})^{4} = v_{i}$, and replace the square root term in $\tau$ by $(z^{\textnormal{s}}_{i})^{2}$.
 Now we have to deal with the division in layer normalization. For a given $i$ the expression in the numerator is a column vector with coordinates polynomials, say  $w_{1},\dots,w_{d}$ in ${\bar x}$, $\paramvars$ and new variables added in the softmax stage. 
 Add new variables $z^{\textnormal{n}}_{i,j}$ for $j=1,\dots,d$ and  equations  $z^{\textnormal{n}}_{i,j}(z^{\textnormal{s}}_{i})^{2} = w_{j}$. And substitute the column vector with coordinates $z^{\textnormal{n}}_{i,j}$ for the rational expression in layer normalization. 

 Continue the process inductively and cumulatively  through the layers and for the relevant $\tau$.  
 
 For each $i \in [r]$, let $F'_i(\bar{y}, \bar{z}_i)$ be the system of equations resulting from the above rewriting process on $F_i(\bar{y})$, with all the variables in $\bar{z}_i$ renamed to fresh names.
 Let $S(\paramvars, {\bar z}_1,\dots,{\bar z}_r)$ be the union of the $F'_i(\bar{y}, \bar{z}_i)$ for $i \in [r]$.
 Let $\numnewvars$ be the total number of variables in $\bar{z}_i$ for $i \in [r]$.  So the number of variables in the system $S$ is $k+ \numnewvars$ and the number of equations is $r+\numnewvars$.

As in the proof of \citet[p.~173--174]{KM1997} we have the following claims:
\begin{enumerate}[label=(\roman*)]
\item Suppose $S(\paramvars, {\bar z}_{F_{1}},\dots,{\bar z}_{F_{r}})$, then $F_{i}(\paramvars) = \epsilon_{i}$ for $i=1,\dots,r$.
\item  Suppose $F_{i}(\paramvars) = \epsilon_{i}$ for $i=1,\dots,r$ then there are unique $\bar{z}_{F_{i}}$ for $i=1,\dots,r$ such that $S( \paramvars, {\bar z}_{F_{1}},\dots,{\bar z}_{F_{r}})$.
\item The set in $\R^{k}$ defined by $F_{i}(\paramvars) = \epsilon_{i}$ for $i=1,\dots,r$ is homeomorphic to the set in $\R^{k+\numnewvars}$ 
defined by $S(\paramvars,{\bar z}_{F_{1}},\dots,{\bar z}_{F_{r}})$. 
\end{enumerate}

By the claims above and our assumptions on the $F_{i}$, the subset of $\R^{k+\numnewvars}$ defined by $S$ is a manifold of dimension $k-r$. On the other hand $S$ consists of $r+\numnewvars$ equations in $k+\numnewvars$ variables, and $(k+\numnewvars) - (r+\numnewvars) = k-r$.

Since the $\tau_i(\bar{a}_j, \bar{y})$ are exponential--linear polynomials,  \cref{Khovanskii_theorem} applies, to give that the number of connected components of the set defined by $S$ (so also by part (iii) of the Claim, defined by $F_{1}(\paramvars) = \epsilon_{1},\dots,F_{r}(\paramvars) = \epsilon_{r}$) is at most $B$
where $\log B \in O(q^2 + M \log M)$,
where $q$ is the total number of exponentials appearing in $S$,
and $M$ is the maximum of $r+\numnewvars$, $k+\numnewvars$, and $D$ (the maximum degree of polynomials in $S$).
The only quantities in the expression for $B$ that depend on the input are $\numnewvars \in O(n^2)$ and $q$. The transformer has $O(n^2)$ exponentials (\cref{eq:softmax}), but because the exponents $\Lambda_1, \ldots, \Lambda_q$ may depend on $\bar{x}$, and $\bar{x}$ may be substituted with up to $r$ different points $\bar{a}_i$, the number of exponentials must be multiplied by $r \le k$. But $k$ does not depend on $n$, so we still have $q \ino O(n^2)$. Thus $\log B \in O(n^4)$.
Hence, by \cref{thm:km1997}, the VC dimension of $T^{n}$ is at most $2 \log B + (16+2\log s)k$ where $s$ is the number of $C^{\infty}$ functions appearing in $\phi$ above. 
As $\log s \in O(n)$ and $k$ depends just on the untrained transformer $T$, the VC dimension is at most $O(n^4)$, concluding the proof of \cref{thm:transformer_vc}.

\section{Transformers computing symmetric functions}
\label{sec:transformer-constructions}

\begin{proposition} \label{thm:symmetric_unary}
For any family of (symmetric) functions $f^{(n)} \colon \{\sym{1}\}^n \to \{ 0, 1 \}$, there exists a trained transformer $T_\paramvals$ such that, for any $n$, $T_\paramvals^{(n)}(\sym{1}^n) = f^{(n)}(\sym{1}^n)$.
\end{proposition}

\begin{proof}
We don't need a word embedding, since there is only one symbol. We use a positional encoding $\textnormal{PE}(i) = (i, f^{(i)}(\sym{1}^i))$. The self-attention layer just has to retrieve the last coordinate of the positional encoding at the last position $n$. For each vector $X_{i,*}$, we set the weights such that:
\begin{align*}
W^{\textnormal{Q}} X_{i,*} &= \begin{bmatrix} \gamma \end{bmatrix} &
W^{\textnormal{K}} X_{j,*} &= \begin{bmatrix} j \end{bmatrix} &
W^{\textnormal{V}} X_{j,*} &= \begin{bmatrix} f^{(j)}(\sym{1}^j) \end{bmatrix}.
\end{align*}
At position $i=n$, the score $s_{nj}$ is maximized for $j=n$, and the scores decrease by $\gamma$ for each position to the left. 
By Lemma 6 of \citet{yang-etal-2025-softmax}, there is a choice of $\gamma$ that ensures that the attention on position $n$ is greater than $3/4$, so we can make the attention output positive iff \mbox{$f^{(n)}(\sym{1}^n) = 1$.} If there is layer normalization, then we can use the technique of \citet[Sec.~4.1]{chiang-cholak-2022-overcoming} so that layer normalization does not change the sign of the attention output.
\end{proof}

\begin{proposition}
    \label{prop:symmetric-size-two}
    There exists a trained transformer $T_\paramvals$ such that, for any $n>0$ and symmetric function $f \colon \{ \sym{0}, \sym{1} \}^n \to \{ 0, 1 \}$, there exists a positional encoding $\textnormal{PE} \colon [n] \to \mathbb{R}^d$ such that $T^{(n)}_\paramvals$ with positional encoding $\textnormal{PE}$ computes $f$.
\end{proposition}

\begin{proof}
    For each symmetric function $f \colon \{ \sym{0}, \sym{1} \}^n \to \{ 0, 1 \}$, there exists by definition a function $\bar{f} \colon [n] \to \{ -1 , 1 \}$ such that, for all (or, equivalently, any) string $w$ that has exactly $i$ many $\sym{1}$s, we have $\bar{f}(i) = 1$ if $f(w) = 1$ and $\bar{f}(i) = -1$ if $f(w) = 0$. (Since we are working over a binary alphabet with strings of a fixed length, the number of $\sym{1}$s determines the number of $\sym{0}$s.) To construct a transformer computing $f$, we follow verbatim the construction of \citet{chiang-cholak-2022-overcoming} of a transformer computing the $\mathsf{PARITY}$ function (Section 3.2 of their paper), but with one change. Where they include $\cos(i \pi)$ (equivalently, $(-1)^i$) in the final coordinate of their positional encoding, we instead include $\bar{f}(i)$. For the same reason that their construction computes the $\mathsf{PARITY}$ function, ours computes $f$. All of the transformers constructed in this way use the exact same learned parameters—their only difference is that they use different positional encodings.
\end{proof}

\begin{proposition} \label{thm:undecidable}
$L$ is recognized by a transformer, where $\Omega = \sum_{x : \text{$U$ halts on $x$}} 2^{-|x|}$ and $L = \{ w \in \{ \sym{0}, \sym{1} \}^* \mid \#_\sym{1}(w) / \vert{}w\vert{} < \Omega \}$.
\end{proposition}
\begin{proof}
We construct a transformer with one layer. The self-attention attends uniformly ($W^{\textnormal{Q}} = W^{\textnormal{K}} = \mathbf{0}$).  The values are $\mathbb{I}[w_i = \sym{1}]$.
The FFNN computes
$\mathit{ffn}(x) = \max(0, \Omega-x)$.
The transformer accepts the string if $\mathit{ffn}(x) > 0$, which occurs if and only if $\#_{\sym{1}}(w)/\vert{}w\vert{} < \Omega$.
\end{proof}

\section{Proof of \cref{prop:hard-coded-index}}
\label{sec:hard-coded-index-proof}

The key argument is that extracting a deeper bit requires the output to alternate increasingly many times as $x$ varies. We first bound how many connected components a fixed transformer can produce.

\begin{lemma}
\label{lem:hard-coded-components}
Let $T$ be an untrained transformer. For any parameter values $\paramvals$, length $n$, and sequence $v \in \R^{(n-1)}$, the set \[
S_{\paramvals,v}^{(n)}
=
\left\{
x\in[0,1):
T_\paramvals^{(n)}
\!\left(
x, v
\right)>0
\right\}.
\] has $2^{O(n^4)}$ connected components. %
\end{lemma}

\begin{proof}
Fix $n$, $\paramvals$, and $v$. 
We want to apply \cref{Khovanskii_theorem} directly to bound the number of connected components of $S_{\paramvals,v}^{(n)}$, so we need to rewrite the transformer in a suitable form. 

First, we need to eliminate square roots and divisions, and to ensure that the argument of every $\exp$ is linear, just as in the proof of \cref{thm:transformer_vc}.
Second, we need to eliminate $\max$, and third, we need to replace inequalities with equalities.

For each ReLU $z=\max(0,h)$, use the equivalent conjunction
\[
z\geq 0,\qquad z-h\geq 0,\qquad z(z-h)=0.
\]
Then for each inequality, introduce a new real variable $u$ and use
\[
p\geq 0\ \Longleftrightarrow\ (\exists u)\ p-u^2=0.
\]
The resulting system is logically equivalent to the original transformer computation.

Let $\bar z$ denote the tuple consisting of all variables. We obtain exponential--linear polynomials $P_1(x,\bar z),\ldots,P_s(x,\bar z)$ such that
\begin{equation}
\label{eq:hard-coded-projection}
T_\paramvals^{(n)}\!\left(x, v\right)>0
\quad\Longleftrightarrow\quad
(\exists\bar z)\ \bigwedge_{i=1}^{s} P_i(x,\bar z)=0.
\end{equation}
For a fixed transformer, the above representation has $O(n^2)$ computation variables and constraints, $O(n^2)$ exponential terms, and polynomial degrees bounded by $O(1)$. The replacement of inequalities with equations adds $O(n^2)$ variables and equations, without introducing new exponential terms. 
Thus, let $\numnewvars=|\bar z|$, let $r$, $q$, and $D$ denote respectively the number of equations, the number of exponential terms, and the maximum degree of the exponential--linear polynomials $P_i$. Then we have
\[
\numnewvars,r,q\in O(n^2),\qquad D\in O(1).
\]
To apply \cref{Khovanskii_theorem}, the common solution set in \cref{eq:hard-coded-projection} must be a smooth manifold, but this may not be the case. For example, the equations used above to encode a ReLU relation $z=\max(0,h)$ become
\[
z-u^2=0,\qquad z-h-w^2=0,\qquad z(z-h)=0.
\]
Their solution set is the union of the two curves
\[
\{(-w^2,0,0,w):w\in\mathbb R\}
\quad\text{and}\quad
\{(u^2,u^2,u,0):u\in\mathbb R\},
\]
where the coordinates are ordered as $(h,z,u,w)$. These curves meet at the origin with distinct tangent directions, so there is no single tangent line to the solution set there. Thus the solution set is not a smooth manifold at the origin. 
To get around this issue, we construct a nearby regular level set, which, by the regular level set theorem, will allow us to apply \cref{Khovanskii_theorem}. Then $\epsilon$ is a regular value of $F$, and the regular level set theorem implies that $F^{-1}(\epsilon)$ is a smooth sub-manifold of $\mathbb{R}^{\numnewvars+2}$ of co-dimension one, and hence of dimension $\numnewvars+1$.

We begin by combining the equations into a single nonnegative exponential--linear polynomial
\[
F_0(x,\bar z)
=
\sum_{j=1}^{s} P_j(x,\bar z)^2.
\]
Since every term is nonnegative,
\[
F_0(x,\bar z)=0
\quad\Longleftrightarrow\quad
P_1(x,\bar z)=\cdots=P_s(x,\bar z)=0.
\]
Choose any finite number $N\geq 2$ of points $a_1<\cdots<a_N$ in distinct connected components of $S_{\paramvals,v}^{(n)}$. 
For each $i=1,\ldots,N-1$, choose $c_i\in(a_i,a_{i+1})$ not in $S_{\paramvals,v}^{(n)}$. By \cref{eq:hard-coded-projection}, each $a_i$ has a tuple $\bar z_i$ satisfying \[
P_1(a_i,\bar z_i)=\cdots=P_r(a_i,\bar z_i)=0,
\] whereas no such tuple exists at any $c_i$.

The next step is to write a modified function $F$ and choose $\epsilon$ so that the set $F^{-1}(\epsilon)$ does not contain any point whose $x$ coordinate equals to any of the intermediate values $c_i$. So we need a strictly positive lower bound for $F$ when $x=c_i$. We therefore modify $F_0$ as follows:
Choose $R>0$ such that $a_i^2+\|\bar z_i\|^2<R^2$ for all $i$.  Introduce one more real variable $t$, and define
\begin{equation}
\label{eq:hard-coded-regularization}
F(x,\bar z,t)
=
\sum_{i=1}^{s} P_i(x,\bar z)^2
+
\bigl(x^2+\|\bar z\|^2+t^2-R^2\bigr)^2.
\end{equation}
The variable $t$ is a slack variable: the equation
$x^2+\|\bar z\|^2+t^2=R^2$ is solvable in $t$ exactly when
\
$x^2+\|\bar z\|^2\le R^2$. Thus, at level zero, the final squared term restricts $(x,\bar z)$ to the closed ball of radius $R$ about the origin while expressing this restriction using an equation instead of an inequality.

This resulting function has three properties we need. First, for each $a_i$ there are values $\bar z_i$ and $t_i$ such that $F(a_i,\bar z_i,t_i)=0$. Second, for each $c_i$, we have that $\delta_i
:=
\min_{\bar z,t} F(c_i,\bar z,t)
>0$, meaning the function $F(c_i,\bar z,t)$ is bounded away from zero (while $F_0(c_i,\bar z)$ was not, as $\bar z$ does not necessarily range over a compact domain). Finally, $F$ is still an exponential--linear polynomial, with the same exponential terms as the $P_i$ and degree $O(1)$.

Now, chose a value $0<\epsilon<\min_{1\leq i\leq N-1}\delta_i$ and inputs $x \in \R$ , $\bar{z} \in \R^{\numnewvars}$, and $t \in \R$ such that $F(x, \bar{z}, t) = \epsilon$ and $\nabla F(x,\bar{z},t)\neq 0$ (in other words, $\epsilon$ is a regular value of $F$).'' Such a choice is possible by Sard's theorem. Then $F^{-1}(\epsilon)$ is a smooth manifold. Since $F(a_i,\bar z_i,t_i)=0$ and $F(a_i,\bar z_i,t)\to\infty$ as $|t|\to\infty$, continuity shows that $F^{-1}(\epsilon)$ contains a point with $x=a_i$ for every $i$.
On the other hand, by the choice of $\epsilon$, it contains no point with $x=c_i$. Hence these $N$ points lie in distinct connected components of $F^{-1}(\epsilon)$.

The equation $F-\epsilon=0$ is in $k=\numnewvars+2$ variables: the $\numnewvars$ variables in $\bar z$, together with $x$ and the additional variable $t$. The equation $F-\epsilon=0$ has $k=\numnewvars+2$ variables and $r=1$ equation,
so
\[
\dim F^{-1}(\epsilon)=\numnewvars+1=k-r.
\]
We can now apply \cref{Khovanskii_theorem} to the equation $F-\epsilon=0$. It gives at most
\[
2^{O(q^2+k\log k)} \in 2^{O(n^4)}
\]
connected components. The implied constant depends only on the fixed transformer $T$.
Hence,
\[
N \ino 2^{O(n^4)}
\]
and the number of connected components must be at least $N$.
The cases of zero or one component are trivial.
\end{proof}
\begin{proof}[Proof of \cref{prop:hard-coded-index}]
Suppose that $T_\paramvals$ solves the $b(n)$-th-bit index task. Fix any $v$.
On the intervals
\[
\bigl(j2^{-b(n)},(j+1)2^{-b(n)}\bigr),
\qquad j=0,\ldots,2^{b(n)}-1,
\]
the $2^{-b(n)}$ bit alternates between $0$ and $1$. Hence
$S_{\paramvals,v}^{(n)}$ has at least $2^{b(n)-1}$ connected components.
By the preceding lemma,
\[
2^{b(n)-1} \ino 2^{O_T(n^4)}.
\]
Thus $b(n)\in O_T(n^4)$. Thus, the proposition holds.
\end{proof}

\end{document}